\documentclass[sigconf]{acmart}

\usepackage{algorithm}
\usepackage{algorithmic}
\usepackage{amsmath}
\usepackage{amsthm}

\usepackage{booktabs}
\usepackage{siunitx}
\usepackage{multirow}

\usepackage{booktabs}    
\usepackage{multirow}   
\usepackage{makecell}   
\usepackage{caption}    
\usepackage{graphicx}   

\usepackage{booktabs}
\usepackage{multirow}

\newcommand{\std}[1]{_{\pm \text{\scriptsize #1}}}

\newtheorem{theorem}{Theorem}
\newtheorem{lemma}[theorem]{Lemma}
\newtheorem{proposition}[theorem]{Proposition}

\AtBeginDocument{%
  }

\setcopyright{acmlicensed}
\copyrightyear{2018}
\acmYear{2018}
\acmDOI{XXXXXXX.XXXXXXX}
\acmConference[Conference acronym 'XX]{Make sure to enter the correct
  conference title from your rights confirmation email}{June 03--05,
  2018}{Woodstock, NY}
\acmISBN{978-1-4503-XXXX-X/2018/06}

\renewcommand\footnotetextcopyrightpermission[1]{}

\begin{document}

\title{Topology-Adaptive Hyperbolic Graph Attention Networks Guided by the Hyperbolic Sombor Index}


\author{Haifang Cao}
\affiliation{%
  \institution{Tianjin University}
  \state{Tainjin}
  \country{China}}
\email{caohaifang@tju.edu.cn}

\author{Boan Tao}
\affiliation{%
  \institution{Tianjin University}
  \state{Tainjin}
  \country{China}}

\author{Xiyuan Gao}
\affiliation{%
  \institution{Tianjin University}
  \state{Tainjin}
  \country{China}}

\author{Timing Li}
\affiliation{%
  \institution{Tianjin University}
  \state{Tainjin}
  \country{China}}

\author{Yu Wang}
\affiliation{%
  \institution{Tianjin University}
  \state{Tainjin}
  \country{China}}

\author{Pengfei Zhu}
\affiliation{%
  \institution{Tianjin University}
  \state{Tainjin}
  \country{China}}

\renewcommand{\shortauthors}{Trovato et al.}

\begin{abstract}
  Hyperbolic geometry has emerged as a principled space for representing hierarchical graphs. However, existing hyperbolic graph neural networks typically rely on shared curvature configurations and feature-driven attention, failing to explicitly exploit local hierarchical topological patterns. To bridge this gap, we introduce the Hyperbolic Sombor Index (HSO) as a lightweight structural prior for capturing hierarchy-indicative degree stratification. Building on this, we propose \textbf{HSO-GAT}, a topology-adaptive hyperbolic graph attention network that unifies geometric adaptation and message propagation. Specifically, it comprises two complementary modules: HSO-Guided Local Curvature Adaptation, which performs adaptive node-wise geometric scaling from aggregated node-level HSO signals, and HSO-Gated Hyperbolic Graph Attention, which enables structure-aware message passing through feature-conditioned gating. Theoretically, we establish the monotonic sensitivity of edge-level HSO to degree imbalance and analyze the validity and radial scaling properties of node-adaptive hyperbolic mappings. Extensive experiments on eight benchmark datasets demonstrate that HSO-GAT consistently achieves state-of-the-art performance in both node classification and link prediction tasks.
\end{abstract}





\maketitle

\pagestyle{plain}

\section{Introduction}
\label{sec:intro}

Real-world graphs frequently exhibit hierarchical, tree-like structures, which can incur substantial distortion when represented in Euclidean space because of its polynomial volume growth \cite{bronstein2017geometric, chen2025hyperbolic}. Consequently, hyperbolic geometry, with its exponential volume expansion, has emerged as the principled space for hierarchical graph representation. Although hyperbolic graph neural networks (HGNNs) \cite{chami2019hyperbolic, liu2019hyperbolic, grover2025spectro, cao2025hyperbolic} have been developed to leverage this non-Euclidean space, existing approaches, including hyperbolic graph attention networks \cite{zhang2021hyperbolic}, still fail to fully exploit the intricate local hierarchical patterns of the underlying graph topology.

This gap manifests in two critical limitations. First, existing HGNNs predominantly rely on globally or layer-wise shared curvature \cite{peng2021hyperbolic, wang2025hyperbolic}. Nodes in dense cores, intermediate regions, and expanding peripheral branches may exhibit distinct geometric requirements that cannot be accommodated by shared curvature configurations. Although recent studies have explored more flexible curvature modulation \cite{fu2021ace,yang2023kappa,guo2025graphmore,cao2026geometric}, explicit hierarchy-sensitive guidance for node-wise geometric adaptation remains underexplored. Second, graph attention is primarily driven by learned feature compatibility \cite{velivckovic2017graph,zhang2021hyperbolic}, leaving the structural roles of individual connections only implicitly modeled. Thus, feature-similar edges with different hierarchical roles may be treated similarly. Fundamentally, both limitations stem from the lack of a lightweight, hierarchy-sensitive topological prior that jointly guides local geometric adaptation and neighborhood aggregation.


\begin{figure}[t!]
  \centering
  \includegraphics[width=0.99\columnwidth]{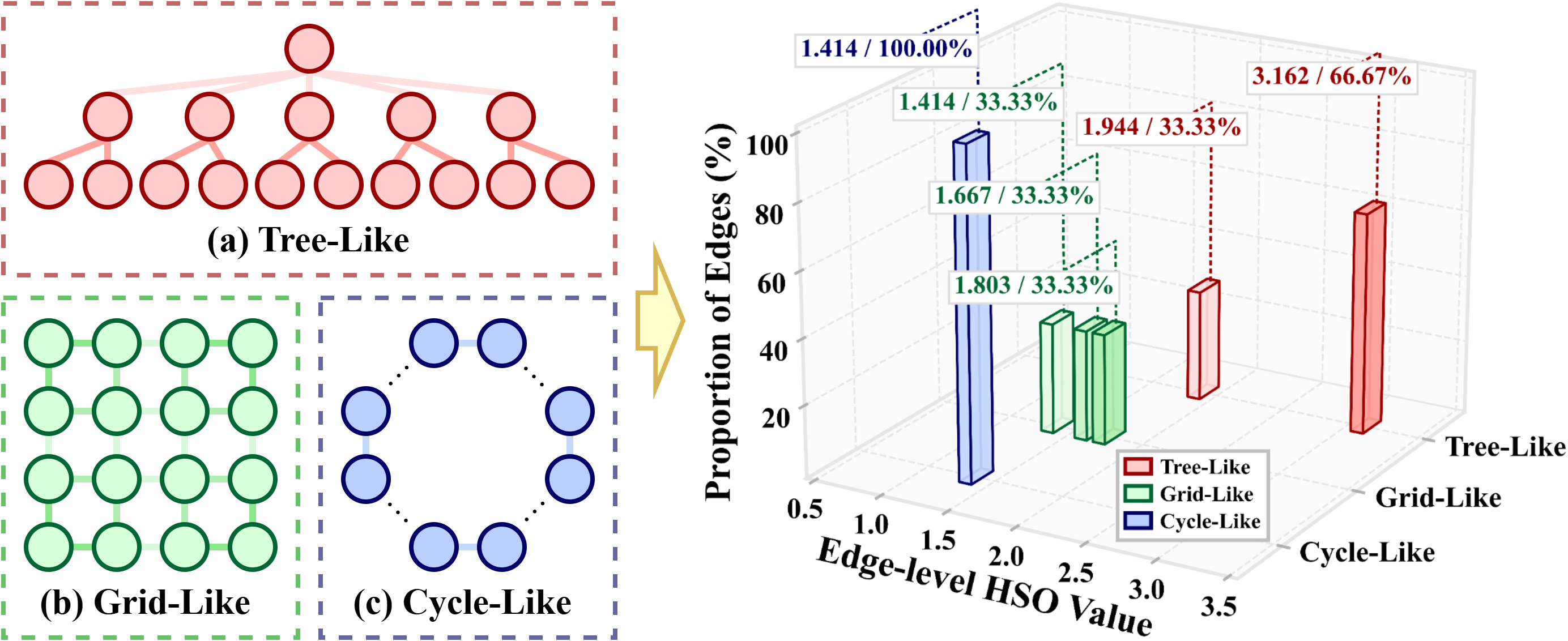}
 \caption{Topological sensitivity of the edge-level Hyperbolic Sombor Index (HSO). (Left) Three representative graph structures: (a) Tree-Like, (b) Grid-Like, and (c) Cycle-Like. (Right) The distribution of edge-level HSO values and their proportions across different structures, where labels (e.g., 3.162 / 66.67\%) denote the specific HSO value and the corresponding percentage of edges. Tree-like graphs exhibit higher and broader distributions due to degree imbalances.}
  \label{fig: introduction}
\end{figure}

To meet this need, we turn to the Hyperbolic Sombor Index (HSO) \cite{barman2026geometric,das2025hyperbolic}, a degree-based descriptor originally developed in chemical graph theory. At the edge level, HSO increases monotonically with the degree imbalance between its endpoint nodes. This property assigns larger values to hub-periphery connections while yielding lower values for degree-balanced edges, making HSO sensitive to the degree stratification that characterizes hierarchical organization. As illustrated in Figure~\ref{fig: introduction}, tree-, grid-, and cycle-like graphs with the same number of nodes exhibit distinct edge-level HSO distributions, with the hierarchical tree showing higher values and a broader distribution compared to grid-like and cycle-like graphs. 
This observation motivates HSO as a prior for hierarchy-indicative degree stratification rather than as a universal measure of graph hierarchy. Despite these favorable characteristics, existing studies mainly treat HSO as a graph-level analytical index and molecular descriptor \cite{barman2025chemical,barman2026m}, while its potential as a fine-grained structural prior for deep graph representation learning remains largely unexplored.

Building on this insight, we propose \textbf{HSO-GAT}, a topology-adaptive hyperbolic graph attention network guided by the Hyperbolic Sombor Index. Here, \emph{topology-adaptive} does not refer to modifying the graph connectivity; instead, it denotes adapting geometric modeling and message propagation according to the hierarchy-indicative topological patterns characterized by HSO. The central idea is to exploit HSO at two complementary granularities: edge-level HSO characterizes the structural role of individual connections, while its node-wise aggregation summarizes the local hierarchical context surrounding each node. In this way, HSO-GAT establishes a unified pathway from graph topology to hyperbolic geometry and attention-based propagation. 

Specifically, HSO-GAT comprises two complementary components. First, the \emph{HSO-Guided Local Curvature Adaptation} (HLCA) aggregates edge-level HSO to construct a node-wise structural descriptor. By modulating a learnable global reference curvature with this local signal, HLCA enables topology-dependent geometric scaling, allowing nodes at different hierarchical positions to be represented at local geometric scales. Second, the \emph{HSO-Gated Hyperbolic Graph Attention} (HSGA) injects edge-level HSO into tangent-space message passing through feature-conditioned gating and head-specific control, adaptively incorporating hierarchy-indicative cues without overriding semantic compatibility. 

Our main contributions are summarized as follows:

\begin{itemize}
    \item We pioneer the incorporation of the Hyperbolic Sombor Index (HSO) into deep graph learning, establishing it as a lightweight structural prior that explicitly captures hierarchy-indicative degree stratification.
    \item We propose HSO-GAT, a topology-adaptive framework that combines HLCA for node-wise curvature modulation with HSGA for hierarchy-aware attention, thereby unifying geometric adaptation and message propagation under a shared HSO prior.
    \item We theoretically analyze the monotonic sensitivity of edge-level HSO to degree imbalance and the validity and radial scaling properties of node-adaptive hyperbolic mappings, with extensive experiments demonstrating state-of-the-art performance across eight benchmarks.
\end{itemize}

\section{Related Work}
\label{sec:related_work}

\subsection{Hyperbolic Graph Neural Networks}
Hyperbolic graph neural networks (HGNNs) exploit the exponential volume growth of hyperbolic space to represent hierarchical, power-law, and tree-like graph structures. Foundational methods include HGCN~\cite{chami2019hyperbolic}, which introduces hyperbolic graph convolution via tangent-space transformations, and HAT~\cite{zhang2021hyperbolic}, which extends attention-based aggregation to hyperbolic space. Building on these foundations, subsequent studies have explored structure-adaptive hyperbolic geometries. ACE-HGNN~\cite{fu2021ace} learns task-dependent curvature through reinforcement learning, whereas $\kappa$HGCN~\cite{yang2023kappa} couples discrete graph curvature with continuous manifold curvature. More recently, GraphMoRE~\cite{guo2025graphmore} and ARGNN~\cite{wang2026adaptive} enable finer-grained geometric adaptation: the former constructs node-personalized mixed-curvature spaces, while the latter learns a continuous anisotropic Riemannian metric tensor field. Concurrently, hyperbolic attention has continued to evolve: HHGAT~\cite{park2024hyperbolic} incorporates meta-path instances into hyperbolic attention to capture relational semantics, whereas HypHGT~\cite{park2026hyperbolic} employs relation-specific hyperbolic attention to model both local and global dependencies.

Despite these advances, geometric adaptation and attention modeling have largely been studied in isolation, without a unified structural prior that jointly guides both processes. HSO-GAT addresses this gap by leveraging the Hyperbolic Sombor Index (HSO) as a lightweight, degree-imbalance-aware topological descriptor that simultaneously modulates node-wise curvature and structure-aware attention.

\subsection{Hyperbolic Sombor Index} 

The Hyperbolic Sombor Index (HSO) was introduced as a degree-based topological descriptor motivated by the geometric properties of hyperbolas, with the original study examining its predictive power, structure sensitivity, and degeneracy on hydrocarbon isomers~\cite{barman2026geometric}. The practical relevance of HSO was subsequently investigated through molecular property prediction of benzenoid hydrocarbons~\cite{barman2025chemical}. 
Further studies have expanded its theoretical foundations, ranging from deriving bounds and analyzing extremal properties for various graph classes~\cite{das2025hyperbolic,li2025hyperbolic,rada2026extremal,albalahi2025hyperbolic}, to developing M-polynomial formulations for chemical graph families~\cite{barman2026m}.

However, existing studies have primarily treated HSO as a graph-level analytical index or a molecular descriptor, without exploring its potential for representation learning. To the best of our knowledge, this work is the first to incorporate HSO into deep graph learning as an explicit topological prior.

\begin{figure*}[ht]
    \centering
    \includegraphics[width=0.95\textwidth]{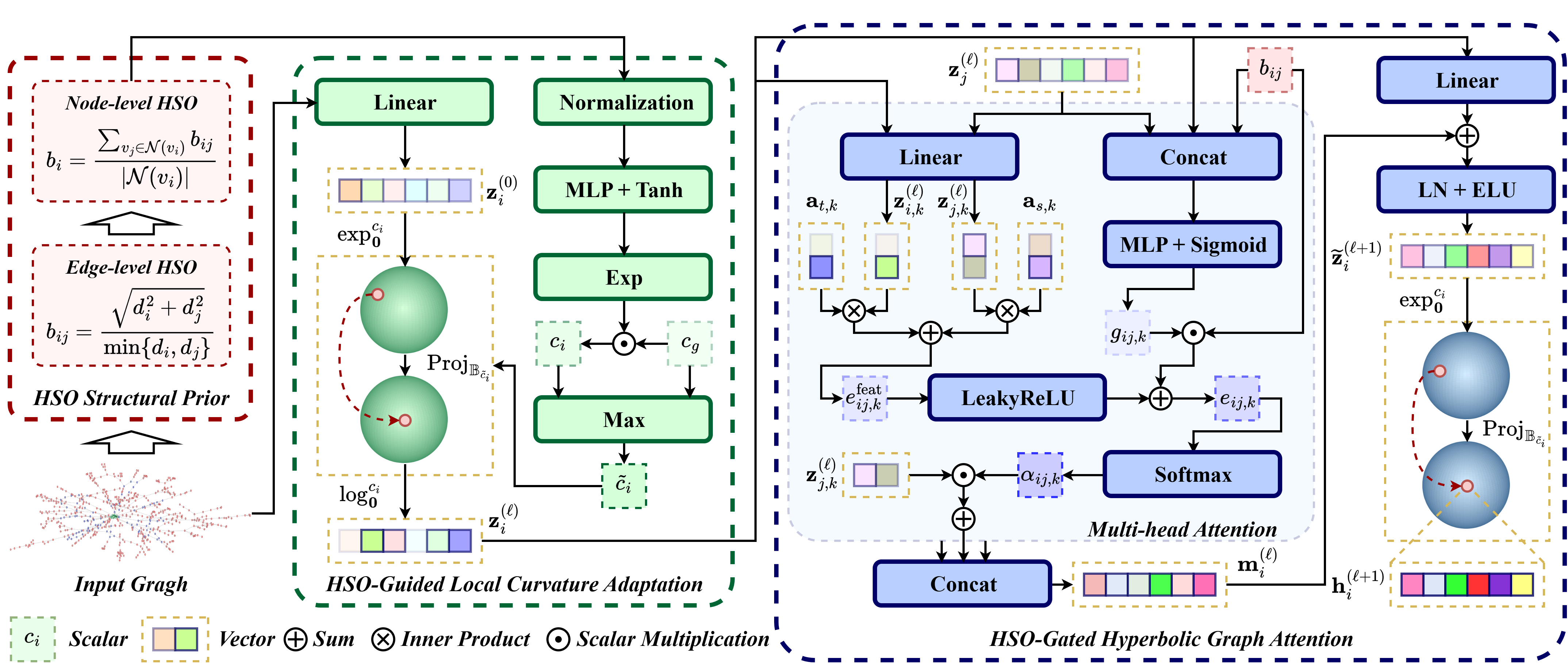}
    \caption{
    Overview of the proposed HSO-GAT framework. Edge-level HSO $b_{ij}$ captures local degree imbalance, and its neighborhood aggregation yields the node-level HSO $b_i$. The HLCA module transforms $b_i$ into node-wise curvature $c_i$ to guide adaptive hyperbolic mappings while keeping representations in a shared Poincar\'e ball. The HSGA module performs multi-head attention in the tangent space, combining feature compatibility with a gated edge-level HSO bias. By coupling HLCA and HSGA, HSO-GAT enables topology-adaptive geometric modeling and structure-aware message propagation.
    }
    
    \label{figure: framework}
\end{figure*}

\section{Preliminary}
\label{sec:preliminary}

\subsection{Notations}

A graph is formally defined as $\mathcal{G}=(\mathcal{V},\mathcal{E})$, where 
$\mathcal{V}=\{v_1,v_2,\ldots,v_N\}$ is the set of $N$ nodes and 
$\mathcal{E}\subseteq \mathcal{V}\times\mathcal{V}$ denotes the set of $m$ edges. 
The connectivity of $\mathcal{G}$ is encoded by the adjacency matrix 
$\mathbf{A}\in\{0,1\}^{N\times N}$, where $A_{ij}=1$ if 
$(v_i,v_j)\in\mathcal{E}$ and $A_{ij}=0$ otherwise. Each node $v_i$ is associated 
with a feature vector $\mathbf{x}_i\in\mathbb{R}^{d_{in}}$, and these feature 
vectors are stacked into the feature matrix $\mathbf{X}\in\mathbb{R}^{N\times d_{in}}$. 
The neighborhood of $v_i$ is denoted by 
$\mathcal{N}(v_i)=\{v_j\mid (v_i,v_j)\in\mathcal{E}\}$, and the degree of $v_i$ 
is defined as $d_i=|\mathcal{N}(v_i)|$.


\subsection{Hyperbolic Geometry}

Hyperbolic space is a non-Euclidean space with constant negative curvature, which is suitable for modeling hierarchical and tree-like structures. 
In this paper, we adopt the Poincar\'e ball model as the underlying hyperbolic 
space. Given a positive curvature parameter $c>0$, the $d$-dimensional 
Poincar\'e ball with curvature $-c$ is defined as
\begin{equation}
\mathbb{B}_c^d = \left\{\mathbf{x}\in\mathbb{R}^{d}: c\|\mathbf{x}\|^2 < 1\right\},
\end{equation}
where $\|\cdot\|$ denotes the Euclidean norm. 

Since standard Euclidean operations are not directly compatible with the geometry of the Poincar\'e ball, hyperbolic graph neural networks typically perform transformations in the tangent space, which serves as a local Euclidean approximation of the manifold. We adopt the tangent space at the origin $\mathcal{T}_{\mathbf{0}}\mathbb{B}_c^d$, and the exponential and logarithmic maps define the transformations between this tangent space and the hyperbolic space. Specifically, the exponential map $\exp_{\mathbf{0}}^{c}: \mathcal{T}_{\mathbf{0}}\mathbb{B}_c^d \rightarrow \mathbb{B}_c^d$ projects tangent vectors onto the Poincar\'e ball, whereas the logarithmic map $\log_{\mathbf{0}}^{c}: \mathbb{B}_c^d \rightarrow \mathcal{T}_{\mathbf{0}}\mathbb{B}_c^d$ maps hyperbolic embeddings back to the tangent space. For $\mathbf{v}\in\mathcal{T}_{\mathbf{0}}\mathbb{B}_c^d$ and $\mathbf{x}\in\mathbb{B}_c^d$, the two maps are defined as

\begin{equation}
\exp_{\mathbf{0}}^{c}(\mathbf{v}) = \tanh\left(\sqrt{c}\|\mathbf{v}\|
\right)\frac{\mathbf{v}}{\sqrt{c}\|\mathbf{v}\|},
\end{equation}
\begin{equation}
\log_{\mathbf{0}}^{c}(\mathbf{x}) = \operatorname{artanh}\left(\sqrt{c}\|
\mathbf{x}\|\right)\frac{\mathbf{x}}{\sqrt{c}\|\mathbf{x}\|}.
\end{equation}

To keep embeddings within the valid Poincar\'e ball, we further apply a projection operator:

\begin{equation}
\mathrm{Proj}_{\mathbb{B}_c}(\mathbf{x}) = \frac{\mathbf{x}}{\max\left(1, \frac{\sqrt{c}\|\mathbf{x}\|}{1-\epsilon}\right)},
\end{equation}
where $\epsilon>0$ is a small constant. 
This operation prevents embeddings from crossing or becoming overly close to the boundary of the Poincar\'e ball.

\subsection{Hyperbolic Sombor Index}\label{Hyperbolic Sombor Index}
Inspired by the Hyperbolic Sombor Index (HSO), a degree-based topological descriptor originally developed in chemical graph theory~\cite{barman2026geometric, das2025hyperbolic}, we extract its fundamental edge-wise component to explicitly capture local degree imbalance. 
For an edge $(v_i,v_j)\in\mathcal{E}$, the edge-level HSO is defined as:
\begin{equation}
{b}_{ij} = \frac{\sqrt{d_i^2+d_j^2}}{\min\{d_i,d_j\}},
\end{equation}
where $d_i$ and $d_j$ denote the degrees of nodes $v_i$ and $v_j$, respectively. The edge-level HSO value characterizes the degree imbalance between two incident nodes. 
Such degree-stratified patterns are closely associated with hierarchical organizations in complex graphs.


\section{Method}

\subsection{Overall Framework}
To overcome the limitations of globally shared curvature and feature-driven attention in conventional hyperbolic graph neural networks, we  propose \textbf{HSO-GAT}. The central idea is to leverage the Hyperbolic Sombor Index (HSO) as a local hierarchy-aware structural prior, which captures degree-imbalance and hub-induced expansion patterns, to guide both node-wise curvature modulation and structure-aware attention propagation. HSO-GAT consists of two coupled modules: (i) an HSO-guided local curvature adaptation (HLCA) mechanism that modulates node-wise geometric scales during hyperbolic mapping, and (ii) an HSO-gated hyperbolic graph attention (HSGA) mechanism that injects hierarchy-aware topological priors into neighborhood message passing. The overall architecture is illustrated in Figure~\ref{figure: framework}.

\subsection{HSO-Guided Local Curvature Adaptation}


The shared-curvature assumption in hyperbolic GNNs may overlook hierarchical heterogeneity among nodes. Since nodes at different hierarchical positions require different geometric scales, we leverage HSO as a local hierarchy-aware prior to guide node-wise curvature modulation.

\subsubsection{Node-level HSO Structural Prior}

As defined in the preliminary, the edge-level HSO $b_{ij}$ captures degree disparity. Its validity as a hierarchical prior is theoretically supported by its monotonic sensitivity to local degree imbalance:

\begin{lemma}[Hierarchical Monotonicity of Edge-level HSO]
\label{lemma:hso_monotonicity}
For any edge $(v_i,v_j)\in\mathcal{E}$ with $d_i\ge d_j$, let the degree ratio be $r = d_i/d_j \ge 1$. The edge-level HSO $b_{ij} = \sqrt{r^2+1}$ is strictly monotonically increasing with respect to $r$.
\end{lemma}

The proof is provided in Appendix B.1. Lemma~\ref{lemma:hso_monotonicity} guarantees that edges with substantial degree disparities inherently receive higher topological weights. Building upon this theoretical property, we aggregate the edge-level values to obtain the node-level HSO:

\begin{equation}
b_{i} = \frac{\sum_{v_j\in\mathcal{N}(v_i)} b_{ij}}{|\mathcal{N}(v_i)|},
\end{equation}
where $\mathcal{N}(v_i)$ denotes the neighborhood of node $v_i$. Thus, $b_i$ summarizes the degree-imbalance patterns surrounding node $v_i$ and provides stable topological guidance for the subsequent node-wise curvature modulation.




\subsubsection{HSO-driven Node Curvature Modulation}

Given the node-level HSO $b_i$, we first obtain the normalized signal $\bar{b}_i = \frac{b_i - \min_{\mathcal{V}} b}{\max_{\mathcal{V}} b - \min_{\mathcal{V}} b + \epsilon}$ to reduce scale variation, where $\epsilon>0$ ensures numerical stability. The normalized structural signal is then transformed by a lightweight modulation network to define the node-adaptive curvature parameter:
\begin{equation}
c_i = c_g \exp\left(\tanh\left(\mathrm{MLP}\left(\bar{b}_i\right)\right)\right),
\end{equation}
where $c_g>0$ is a learnable global reference curvature parameter. The bounded HSO-conditioned modulation adjusts the local geometric scale around this reference. Consequently, $c_i$ serves as a node-adaptive curvature parameter during hyperbolic mappings rather than defining an independent hyperbolic space.

\subsubsection{Node-adaptive Hyperbolic Mapping}

With the node-adaptive curvature parameter $c_i$, we map Euclidean input features into the Poincar\'e ball through node-adaptive exponential mapping. Given the input feature $\mathbf{x}_i \in \mathbb{R}^{d_{in}}$, its initial tangent-space representation $\mathbf{z}_i^{(0)} \in \mathbb{R}^d$ is
\begin{equation}
\mathbf{z}_i^{(0)} = \mathbf{W}_{\mathrm{in}}\mathbf{x}_i,
\end{equation}
where $\mathbf{W}_{\mathrm{in}} \in \mathbb{R}^{d \times d_{in}}$ is the input transformation matrix.
The initial hyperbolic representation $\mathbf{h}_i^{(0)} \in \mathbb{B}_{c_g}^d$ is obtained through node-adaptive exponential mapping followed by projection onto the shared Poincar\'e ball:



\begin{equation}
\mathbf{h}_i^{(0)} = \mathrm{Proj}_{\mathbb{B}_{\tilde{c}_i}}\left(\exp_{\mathbf{0}}^{c_i}\left(\mathbf{z}_i^{(0)}\right)\right),
\quad \tilde{c}_i = \max\left(c_g, c_i\right).
\end{equation}


The projection operator serves as radial clipping with the effective
curvature $\tilde{c}_i=\max(c_g,c_i)$. Since $\tilde{c}_i \ge c_g$, the
projected representation always lies within the shared Poincar\'e ball
$\mathbb{B}_{c_g}^d$, keeping all nodes in a comparable coordinate domain;
meanwhile, $\tilde{c}_i \ge c_i$ guarantees a uniform margin
$\sqrt{c_i}\|\mathbf{h}_i^{(\ell)}\| \le 1-\epsilon$, so that the subsequent
node-adaptive logarithmic map $\log_{\mathbf{0}}^{c_i}$ is always
well-defined and numerically stable (see Proposition~\ref{prop:log_validity}).

\begin{theorem}[Validity and Radial Metric Scaling of Node-Adaptive Exponential Mapping]
\label{thm:exp_mapping}

For a tangent vector $\mathbf{v}_i \in \mathbb{R}^d$ and positive scaling factor $c_i>0$, the node-adaptive exponential mapping $\hat{\mathbf{h}}_i = \exp_{\mathbf{0}}^{c_i}(\mathbf{v}_i) = \frac{\tanh(\sqrt{c_i}\|\mathbf{v}_i\|)}{\sqrt{c_i}\|\mathbf{v}_i\|} \mathbf{v}_i$ satisfies:
(1) \textbf{Mapping Validity}: $c_i\|\hat{\mathbf{h}}_i\|^2 < 1$;
(2) \textbf{Euclidean Radial Compression}: For fixed $\rho = \|\mathbf{v}_i\| > 0$, the Euclidean radius $R(c_i, \rho) = \|\hat{\mathbf{h}}_i\| = \frac{\tanh(\sqrt{c_i}\rho)}{\sqrt{c_i}}$ is strictly decreasing with respect to $c_i$;
(3) \textbf{Normalized Radial Expansion}: The normalized radius $\tilde{R}(c_i, \rho) = \sqrt{c_i}\|\hat{\mathbf{h}}_i\| = \tanh(\sqrt{c_i}\rho)$ is strictly increasing with respect to $c_i$.
\end{theorem}

The proof is provided in Appendix B.2.
Unlike conventional hyperbolic graph neural networks using a globally shared curvature during exponential mapping, HSO-GAT introduces the node-adaptive curvature parameter $c_i$ into the logarithmic and exponential transformations to adjust the local radial scaling of node representations while preserving a shared hyperbolic embedding space.

For layer-wise propagation, each layer first applies the node-adaptive logarithmic map to project the current hyperbolic representation $\mathbf{h}_i^{(\ell)} \in \mathbb{B}_{c_g}^d$ back to the tangent space for subsequent attention computation, yielding $\mathbf{z}_i^{(\ell)} \in \mathbb{R}^d$:

\begin{equation}
\mathbf{z}_i^{(\ell)} = \log_{\mathbf{0}}^{c_i}\left(\mathbf{h}_i^{(\ell)}\right).
\end{equation}

\begin{proposition}[Domain Validity of Node-Adaptive Logarithmic Mapping]
\label{prop:log_validity}
Let $\tilde{c}_i=\max(c_g,c_i)$ and
$\mathbf{h}_i = \mathrm{Proj}_{\mathbb{B}_{\tilde{c}_i}}\!\left(\exp_{\mathbf{0}}^{c_i}(\mathbf{v}_i)\right)$
for any $\mathbf{v}_i\in\mathbb{R}^d$. Then
(1) $\mathbf{h}_i\in\mathbb{B}_{c_g}^d$, i.e., all node representations
reside in the shared Poincar\'e ball; and
(2) $\sqrt{c_i}\,\|\mathbf{h}_i\| \le 1-\epsilon < 1$, hence
$\log_{\mathbf{0}}^{c_i}(\mathbf{h}_i)$ is well-defined with a uniform
numerical margin $\epsilon$.
\end{proposition}

The proof is provided in Appendix B.3.


\subsection{HSO-Gated Hyperbolic Graph Attention}


Traditional feature-driven attention treats local topology only implicitly. Motivated by the monotonic structural sensitivity of edge-level HSO (Lemma~\ref{lemma:hso_monotonicity}), we design an HSO-gated hyperbolic graph attention mechanism to inject degree-imbalance-aware priors into tangent-space message passing.

\begin{table*}[ht]
\centering
\resizebox{0.95\textwidth}{!}{%
\begin{tabular}{lcccccccc}
\toprule
 \textbf{Dataset} & \textbf{Pubmed} & \textbf{Computer} & \textbf{Photo} & \textbf{WikiCS} & \textbf{Airport} & \textbf{Disease} & \textbf{CS} & \textbf{ogbn-arxiv} \\
\midrule
Nodes     & 19,717  & 13,752  & 7,650   & 11,701  & 3,188   & 1,044   & 18,333  & 169,343    \\
Edges     & 44,338  & 245,861 & 119,081 & 216,123 & 18,631  & 1,043   & 81,894  & 1,166,243  \\
Classes   & 3       & 10      & 8       & 10      & 4       & 2       & 15       & 40         \\
Features  & 500     & 767     & 745     & 300     & 4       & 1,000   & 6,805   & 128        \\
$\delta$ & 3.5 & 1.5 & 2.5 & 1.5 & 1.0 & 0.0 & 2.0 & 2.5 \\
\bottomrule
\end{tabular}%
}
\caption{Statistics of the experimental datasets.}
\label{tab:dataset_stats}
\end{table*}

\subsubsection{Multi-head Feature Attention in Tangent Space}

Following the tangent-space approximation strategy, we formulate multi-head attention on the tangent representations from the node-adaptive logarithmic mapping. For the $k$-th attention head, we apply a head-specific linear transformation:

\begin{equation}
\mathbf{z}_{i,k}^{(\ell)} = \mathbf{W}_{k}^{(\ell)}\mathbf{z}_i^{(\ell)},
\end{equation}
where $K$ is the number of heads,
$d_h=d/K$ is the dimension of each head,
$\mathbf{W}_{k}^{(\ell)}\in\mathbb{R}^{d_h\times d}$ and $\mathbf{z}_{i,k}^{(\ell)}\in\mathbb{R}^{d_h}$ are the head-specific transformation matrix and representation, respectively.
For a target node $v_i$ and its neighbor $v_j$, the standard feature-based attention score is computed as

\begin{equation}
e_{ij,k}^{\mathrm{feat}} = \left\langle\mathbf{a}_{s,k},\mathbf{z}_{j,k}^{(\ell)}\right\rangle
+ \left\langle\mathbf{a}_{t,k},\mathbf{z}_{i,k}^{(\ell)}\right\rangle,
\end{equation}
where 
$\left\langle{\cdot , \cdot}\right\rangle$ denotes inner product,
$\mathbf{a}_{s,k}, \mathbf{a}_{t,k} \in \mathbb{R}^{d_h}$ are learnable parameter vectors for the source and target nodes, respectively. While this score captures semantic compatibility between node features, it lacks awareness of the local degree-imbalance pattern associated with the corresponding edge.

\subsubsection{HSO-Gated Structural Bias}

To make attention computation aware of local topology, we inject the edge-level HSO bias $b_{ij}$. Instead of directly imposing $b_{ij}$ on the attention score, we introduce a feature-conditioned gate to adaptively control its contribution for each edge and attention head:


\begin{equation}
g_{ij,k} = \sigma\left(\mathrm{MLP}_{g}\left[\mathbf{z}_{j,k}^{(\ell)}, \mathbf{z}_{i,k}^{(\ell)}, b_{ij}\right]\right),
\end{equation}
where $\left[{\cdot , \cdot}\right]$ denotes vector concatenation and $\sigma(\cdot)$ is the sigmoid function. 
With this adaptive gating mechanism, the final HSO-enhanced attention score is defined as
\begin{equation}
e_{ij,k} = \mathrm{LeakyReLU}\left(e_{ij,k}^{\mathrm{feat}}\right)+\beta_kg_{ij,k}b_{ij},
\end{equation}
where $\beta_k$ is a learnable head-level coefficient controlling the overall strength of the HSO structural bias. The normalized attention weight is then computed by

\begin{equation}
\alpha_{ij,k} = \frac{\exp(e_{ij,k})}{\sum_{v_r\in\mathcal{N}(v_i)}\exp(e_{ir,k})}.
\end{equation}

\subsubsection{Hyperbolic Message Aggregation and Residual Update}

Using the structurally modulated attention weights, the message $\mathbf{m}_{i,k}^{(\ell)}\in\mathbb{R}^{d_h}$ aggregated by the $k$-th attention head in the tangent space is then formulated as

\begin{equation}
\mathbf{m}_{i,k}^{(\ell)} = \sum_{v_j\in\mathcal{N}(v_i)}\alpha_{ij,k}\mathbf{z}_{j,k}^{(\ell)}.
\end{equation}
The outputs of the $K$ attention heads are concatenated to obtain
$\mathbf{m}_i^{(\ell)}\in
\mathbb{R}^{d}$. 
To stabilize propagation and preserve the center-node representation, we apply a residual connection followed by layer normalization and ELU activation:


\begin{equation}
\widetilde{\mathbf{z}}_i^{(\ell+1)} = \mathrm{ELU}\left(\mathrm{LayerNorm}\left(\mathbf{m}_{i}^{(\ell)}
+\mathbf{W}_{r}^{(\ell)}\mathbf{z}_i^{(\ell)}\right)\right),
\end{equation} 
where $\mathbf{W}_{r}^{(\ell)}\in\mathbb{R}^{Kd_h\times d}$ aligns the residual dimension, and $\widetilde{\mathbf z}_i^{(\ell+1)}\in\mathbb R^{Kd_h}$ denotes the updated tangent-space representation. The updated tangent-space representation is then mapped back to the hyperbolic space through the node-adaptive exponential map followed by projection operator:


\begin{equation}
\mathbf{h}_i^{(\ell+1)} = \mathrm{Proj}_{\mathbb{B}_{\tilde{c}_i}}\left(\exp_{\mathbf{0}}^{c_i}\left(\widetilde{\mathbf{z}}_i^{(\ell+1)}\right)\right).
\end{equation}
This process produces the next-layer hyperbolic representation $\mathbf{h}_i^{(\ell+1)} \in \mathbb{B}_{c_g}^d$ while preserving the topology-adaptive geometric scaling introduced by HSO-GAT.


\subsection{Decoding and Optimization Objective}

\subsubsection{Decoding}
The final-layer representation $\mathbf{h}_i^{(L)}\in\mathbb{B}_{c_g}^d$ is decoded in a
task-specific manner. For node classification, we apply the node-adaptive logarithmic
map $\log_{\mathbf 0}^{c_i}$ followed by a linear classifier, so that $c_i$ determines
the tangent coordinate in which the decision boundary is expressed. For link
prediction, the score of a candidate pair is given by the Fermi--Dirac decoder
$p(i\!\sim\!j)=\big(\exp((d_{\mathbb{B}}(\mathbf h_i^{(L)},\mathbf h_j^{(L)})^2-\tau)/t)+1\big)^{-1}$,
where $d_{\mathbb{B}}$ is the Poincar\'e distance evaluated at the shared curvature
$c_g$. The node-adaptive curvature therefore influences the decoder through the radial
position at which each node is placed in the shared ball.

\subsubsection{Curvature Regularization}

Although node-wise curvature provides stronger local geometric adaptability, unconstrained curvature learning may cause node-wise curvatures to deviate excessively from the global reference, leading to unstable local geometric scales. To prevent this issue, we introduce a curvature regularization term that encourages node-wise curvatures to stay close to the global geometric reference while preserving HSO-guided local variations:
\begin{equation}
\mathcal{L}_{\mathrm{curv}} = \frac{1}{|\mathcal{V}|}\sum_{v_i\in\mathcal{V}}\left(c_i-c_g\right)^2.
\end{equation}

\subsubsection{Overall Objective}

The total loss of the model is formulated as a weighted combination of the task-specific supervised loss and the curvature regularization loss:
\begin{equation}
\mathcal{L} = \mathcal{L}_{\mathrm{task}}+\lambda_{\mathrm{curv}}\mathcal{L}_{\mathrm{curv}},
\end{equation}
where $\lambda_{\mathrm{curv}}$ is a hyperparameter controlling the strength of curvature regularization.

\begin{table*}[ht]
    \centering
    \small 
    \setlength{\tabcolsep}{4pt} 
    
    \begin{tabular}{lcccccccc}
        \toprule
        \multirow{2}{*}{\textbf{Method}} & PubMed & Computer & Photo & WikiCS & Airport & Disease & CS & ogbn-arxiv \\
         & ($\delta=3.5$) & ($\delta=1.5$) & ($\delta=2.5$) & ($\delta=1.5$) & ($\delta=1.0$) & ($\delta=0.0$) & ($\delta=2.0$) & ($\delta=2.5$) \\
        \midrule
        
        GCN (2017)          & $78.45 \std{0.36}$ & $90.17 \std{0.69}$ & $92.84 \std{0.25}$ & $74.29 \std{1.06}$ & $83.19 \std{1.76}$ & $79.73 \std{1.88}$ & $93.10 \std{0.14}$ & $68.36 \std{0.28}$ \\
        GraphSAGE (2017)    & $76.88 \std{0.56}$ & $89.38 \std{0.98}$ & $91.27 \std{2.44}$ & $74.22 \std{1.36}$ & $88.55 \std{0.84}$ & $75.29 \std{2.76}$ & $92.55 \std{0.19}$ & $66.86 \std{0.50}$ \\
        GAT (2018)          & $77.23 \std{0.51}$ & $91.16 \std{0.55}$ & $92.70 \std{0.51}$ & $76.55 \std{0.88}$ & $86.10 \std{1.39}$ & $80.25 \std{0.54}$ & $93.12 \std{0.15}$ & $68.05 \std{0.30}$ \\
        \midrule
        
        HGNN (2019)         & $76.94 \std{1.12}$ & $91.35 \std{0.63}$ & $91.82 \std{0.35}$ & $74.89 \std{0.42}$ & $84.87 \std{2.09}$ & $82.16 \std{1.41}$ & $94.63 \std{0.09}$ & $67.34 \std{0.40}$ \\
        HGCN (2019)         & $77.03 \std{0.49}$ & $90.58 \std{0.21}$ & $93.19 \std{0.29}$ & $76.67 \std{0.44}$ & $88.82 \std{1.66}$ & $89.36 \std{1.02}$ & $94.24 \std{0.17}$ & $65.32 \std{0.25}$ \\
        HAT (2021)          & $75.56 \std{0.67}$ & $90.69 \std{0.37}$ & $93.38 \std{0.45}$ & $77.62 \std{0.49}$ & $89.25 \std{0.96}$ & $90.43 \std{1.12}$ & $94.38 \std{0.19}$ & $68.91 \std{0.26}$ \\
        LGCN (2021)         & $78.08 \std{0.65}$ & $90.28 \std{0.71}$ & $94.21 \std{0.43}$ & $76.58 \std{0.60}$ & $88.53 \std{1.26}$ & $91.15 \std{1.02}$ & $92.26 \std{0.17}$ & OOM \\
        HYBONET (2022)      & $77.66 \std{1.04}$ & $91.93 \std{0.60}$ & $94.55 \std{0.52}$ & $78.07 \std{0.51}$ & $91.85 \std{0.86}$ & $92.03 \std{1.21}$ & $93.99 \std{0.23}$ & OOM \\
        \midrule
        MotifRGC (2024)     & $78.33 \std{0.85}$ & $91.72 \std{0.82}$ & $93.78 \std{0.93}$ & $78.12 \std{0.72}$ & $92.31 \std{1.85}$ & $92.80 \std{1.23}$ & $94.36 \std{0.23}$ & OOM \\
        LResNet (2025)      & $77.24 \std{0.67}$ & $92.33 \std{0.24}$ & $94.58 \std{0.39}$ & $77.98 \std{0.42}$ & $92.39 \std{0.72}$ & $93.29 \std{1.45}$ & $93.84 \std{0.26}$ & $\underline{70.26} \std{0.28}$ \\
        GraphMoRE (2025)    & $\underline{78.65} \std{0.66}$ & $91.29 \std{0.44}$ & $94.07 \std{1.05}$ & $78.27 \std{0.68}$ & $92.76 \std{0.94}$ & $90.62 \std{1.43}$ & $93.15 \std{0.42}$ & OOM \\
        QGT (2026)          & $77.92 \std{0.99}$ & $\underline{92.50} \std{0.77}$ & $93.72 \std{0.23}$ & $75.74 \std{1.00}$ & $\underline{92.93} \std{1.45}$ & $\underline{94.01} \std{1.90}$ & $92.83 \std{0.14}$ & $67.94 \std{0.61}$ \\
        ARGNN (2026)        & $78.43 \std{0.65}$ & $90.27 \std{0.83}$ & $\underline{94.62} \std{0.61}$ & $\underline{79.33} \std{0.40}$ & $91.39 \std{1.93}$ & $90.76 \std{2.35}$ & $\underline{95.47} \std{0.31}$ & OOM \\
        MRiemGNN (2026)     & $77.39 \std{1.38}$ & $91.08 \std{0.35}$ & $94.10 \std{0.35}$ & $78.51 \std{0.43}$ & $92.26 \std{0.59}$ & $91.98 \std{2.12}$ & $94.76 \std{0.16}$ & $66.45 \std{0.41}$ \\
        \midrule
        
        \textbf{HSO-GAT(Ours)}        & $\mathbf{79.81} \std{0.98}$ & $\mathbf{93.46} \std{0.19}$ & $\mathbf{95.79} \std{0.25}$ & $\mathbf{80.02} \std{0.37}$ & $\mathbf{93.70} \std{1.07}$ & $\mathbf{97.86} \std{1.24}$ & $\mathbf{96.10} \std{0.16}$ & $\mathbf{72.36} \std{0.49}$ \\
        \bottomrule
    \end{tabular}
    \caption{Node classification results. We report accuracy ($\%\pm$ standard deviation) on PubMed, Computer, Photo, WikiCS, CS and ogbn-arxiv, and Macro-F1 ($\%\pm$ standard deviation) on Airport and Disease. All scores are averaged over ten runs. The best result in each column is in bold and the runner-up is underlined. OOM denotes out of memory.}
    \label{tab:results_NC}
\end{table*}

\section{Experiments}\label{experiments}
\subsection{Experimental Settings}\label{Experimental Settings}

\textbf{Datasets.}
We evaluate HSO-GAT on eight benchmark datasets spanning diverse domains and structural characteristics. Pubmed \cite{sen2008collective} is a citation network in which nodes represent scientific publications and edges denote citation relationships. Computer and Photo \cite{shchur2018pitfalls} are derived from the Amazon co-purchase graphs, comprising computer and photography-related products as nodes, with edges reflecting frequent co-purchase patterns. WikiCS \cite{mernyei2020wiki} is a Wikipedia-based reference graph whose nodes represent computer science articles and whose edges indicate hyperlinks between them. Airport \cite{chami2019hyperbolic} models the air transportation network, with airports as nodes and commercial flight routes as edges. Disease \cite{chami2019hyperbolic} dataset simulates the disease propagation tree, where node represents a state of being infected or not by SIR disease. CS \cite{shchur2018pitfalls} is a co-authorship network in which nodes denote researchers and edges represent co-authored publications. Finally, ogbn-arxiv \cite{hu2020open} is a large-scale citation graph containing scientific papers and their citation links. Detailed statistics for all datasets are provided in Table \ref{tab:dataset_stats}.

\textbf{Baselines.}
We compare HSO-GAT with three groups of graph learning methods. The representative Euclidean GNNs include GCN \cite{kipf2016semi}, GraphSAGE \cite{hamilton2017inductive}, and GAT \cite{velivckovic2017graph}. The representative hyperbolic GNNs include HGNN \cite{liu2019hyperbolic}, HGCN \cite{chami2019hyperbolic}, HAT \cite{zhang2021hyperbolic}, LGCN \cite{zhang2021lorentzian}, and HYBONET \cite{chen2022fully}. We further compare with recent Riemannian methods, including MotifRGC \cite{sun2024motif}, LResNet \cite{he2025lorentzian}, GraphMoRE \cite{guo2025graphmore}, QGT \cite{le2026pseudo}, ARGNN \cite{wang2026adaptive}, and MRiemGNN \cite{li2026multiplex}.

\textbf{Implementation Details.}
We evaluate HSO-GAT on both node classification and link prediction, using accuracy and Macro-F1 for node classification and ROC-AUC for link prediction. For node classification, we adopt the standard split for Pubmed, random 60\%/20\%/20\% splits for Computer, Photo, and CS, the first official split for WikiCS, a stratified 70\%/15\%/15\% split for Airport, a class-balanced 30\%/10\%/60\% split for Disease, and the official split for ogbn-arxiv. For link prediction, we randomly split the undirected positive edges into training, validation, and test sets with a ratio of 85\%/5\%/10\%, while sampling an equal number of negative edges for each split. Hyperparameters are selected by grid search over learning rates $\{0.001, 0.005, 0.01, 0.02\}$, dropout rates $\{0.0, 0.1, 0.3, 0.4, 0.5, 0.6\}$, weight decay $\{0, 0.0001, 0.0005, 0.001\}$, and the number of hidden layers $\{1,2,3\}$. We report the mean and standard deviation over 10 random seeds. All experiments are implemented in PyTorch and conducted on an NVIDIA GeForce RTX 3090 GPU.

\begin{table*}[ht]
    \centering
    \small 
    \setlength{\tabcolsep}{4pt} 
    
    \begin{tabular}{lcccccccc}
        \toprule
        \multirow{2}{*}{\textbf{Method}} & PubMed & Computer & Photo & WikiCS & Airport & Disease & CS & ogbn-arxiv \\
         & ($\delta=3.5$) & ($\delta=1.5$) & ($\delta=2.5$) & ($\delta=1.5$) & ($\delta=1.0$) & ($\delta=0.0$) & ($\delta=2.0$) & ($\delta=2.5$) \\
        \midrule
        
        GCN (2017)          & $93.59 \std{0.07}$ & $93.38 \std{0.21}$ & $94.60 \std{0.37}$ & $94.19 \std{0.04}$ & $91.99 \std{0.20}$ & $72.20 \std{0.81}$ & $94.67 \std{0.06}$ & $93.27 \std{0.06}$ \\
        GraphSAGE (2017)    & $94.18 \std{0.07}$ & $93.60 \std{0.13}$ & $93.27 \std{0.44}$ & $94.66 \std{0.57}$ & $94.62 \std{0.29}$ & $70.76 \std{1.38}$ & $93.50 \std{0.03}$ & $95.35 \std{0.07}$ \\
        GAT (2018)          & $92.46 \std{0.18}$ & $94.66 \std{0.24}$ & $92.05 \std{1.22}$ & $95.03 \std{0.22}$ & $94.38 \std{0.21}$ & $79.18 \std{1.36}$ & $94.34 \std{0.08}$ & $91.25 \std{0.44}$ \\
        \midrule
        
        HGNN (2019)         & $93.47 \std{0.28}$ & $96.93 \std{0.27}$ & $97.70 \std{0.14}$ & $96.97 \std{0.10}$ & $96.78 \std{0.28}$ & $75.08 \std{1.43}$ & $96.21 \std{0.07}$ & $96.20 \std{0.18}$ \\
        HGCN (2019)         & $95.51 \std{0.03}$ & $97.72 \std{0.10}$ & $97.39 \std{0.02}$ & $98.16 \std{0.05}$ & $97.95 \std{0.11}$ & $89.59 \std{0.88}$ & $97.08 \std{0.07}$ & $97.02 \std{0.05}$ \\
        HAT (2021)          & $96.55 \std{0.04}$ & $96.95 \std{0.02}$ & $98.42 \std{0.02}$ & $97.28 \std{0.03}$ & $98.11 \std{0.11}$ & $88.91 \std{1.79}$ & $97.17 \std{0.11}$ & $\underline{97.82} \std{0.15}$ \\
        LGCN (2021)         & $96.32 \std{0.15}$ & $97.88 \std{0.12}$ & $98.10 \std{0.03}$ & $98.05 \std{0.05}$ & $98.10 \std{0.08}$ & $\underline{96.98} \std{0.62}$ & $95.43 \std{0.27}$ & OOM \\
        HYBONET (2022)      & $96.23 \std{0.17}$ & $98.42 \std{0.15}$ & $98.63 \std{0.12}$ & $98.26 \std{0.03}$ & $96.79 \std{0.31}$ & $94.62 \std{0.93}$ & $95.48 \std{0.42}$ & OOM \\
        \midrule
        MotifRGC (2024)     & $96.15 \std{0.12}$ & $98.35 \std{0.14}$ & $97.25 \std{0.09}$ & $98.05 \std{0.06}$ & $97.10 \std{0.25}$ & $92.51 \std{1.25}$ & $96.95 \std{0.09}$ & OOM \\
        LResNet (2025)      & $95.66 \std{0.11}$ & $\underline{98.54} \std{0.26}$ & $\underline{98.81} \std{0.01}$ & $98.31 \std{0.02}$ & $96.73 \std{0.68}$ & $95.59 \std{0.54}$ & $97.32 \std{0.04}$ & $97.67 \std{0.04}$ \\
        GraphMoRE (2025)    & $95.60 \std{0.35}$ & $97.81 \std{0.10}$ & $98.34 \std{0.09}$ & $98.24 \std{0.11}$ & $97.65 \std{0.28}$ & $92.74 \std{1.32}$ & $96.81 \std{0.20}$ & OOM \\
        QGT (2026)          & $94.15 \std{0.42}$ & $97.77 \std{0.36}$ & $97.36 \std{0.98}$ & $97.21 \std{0.65}$ & $95.99 \std{0.27}$ & $93.36 \std{1.33}$ & $95.37 \std{0.42}$ & $95.94 \std{0.63}$ \\
        ARGNN (2026)        & $\underline{96.67} \std{0.48}$ & $96.13 \std{0.61}$ & $97.63 \std{0.33}$ & $97.33 \std{0.37}$ & $\underline{98.81} \std{0.10}$ & $96.31 \std{1.63}$ & $\underline{98.11} \std{0.08}$ & OOM \\
        MRiemGNN (2026)     & $95.04 \std{0.21}$ & $96.39 \std{0.45}$ & $97.13 \std{0.28}$ & $\underline{98.42} \std{0.04}$ & $97.72 \std{0.32}$ & $96.26 \std{1.04}$ & $96.87 \std{0.12}$ & $97.77 \std{0.06}$ \\
        \midrule
        
        \textbf{HSO-GAT(Ours)} & $\mathbf{98.48} \std{0.08}$ & $\mathbf{98.93} \std{0.04}$ & $\mathbf{99.19} \std{0.04}$ & $\mathbf{99.13} \std{0.02}$ & $\mathbf{99.01} \std{0.21}$ & $\mathbf{98.52} \std{0.39}$ & $\mathbf{98.67} \std{0.03}$ & $\mathbf{99.21} \std{0.07}$ \\
        \bottomrule
    \end{tabular}
    \caption{Link prediction results. ROC-AUC ($\%\pm$ standard deviation) for all datasets. All scores are averaged over ten runs. The best result in each column is in bold and the runner-up is underlined. OOM denotes out of memory.}
    \label{tab:results_LP}
\end{table*}


\begin{table}[ht]
    \centering
    \small 
    \setlength{\tabcolsep}{9pt} 
    \begin{tabular}{lcccc}
        \toprule
        \multirow{2}{*}{\textbf{Method}} & \multicolumn{2}{c}{PubMed} & \multicolumn{2}{c}{Disease} \\
        \cmidrule(lr){2-3} \cmidrule(lr){4-5}
        & NC & LP & NC & LP \\
        \midrule
        Base (HAT)     & 75.56 & 96.55 & 90.43 & 88.91 \\
        w/o HLCA        & 78.54 & 97.96 & 96.01 & 97.39 \\
        w/o HSGA        & 77.66 & 98.27 & 94.89 & 97.64 \\
        w/o $\mathcal{L}_{\mathrm{curv}}$       & 78.83 & 98.39 & 96.57 & 98.11 \\
        w/ Deg. Ratio       & 79.38 & 98.35 & 97.37 & 97.96 \\
        \textbf{HSO-GAT}& \textbf{79.81} & \textbf{98.48} & \textbf{97.86} & \textbf{98.52} \\
        \bottomrule
    \end{tabular}
    \caption{Ablation study on PubMed and Disease datasets. \textit{w/ Deg. Ratio} replaces the HSO prior with the degree ratio of incident edges. The best result in each column is in bold.}
    \label{tab:ablation}
\end{table}


\subsection{Node Classification}\label{Node Classification Result}

Table~\ref{tab:results_NC} presents the node classification results, where accuracy is reported for PubMed, Computer, Photo, WikiCS, CS, and ogbn-arxiv, while Macro-F1 is used for Airport and Disease. HSO-GAT achieves the best performance across all eight datasets, outperforming the strongest baseline by margins of 0.63 to 3.85 percentage points. 

Notably, the most substantial gains occur on Disease (3.85\%) and ogbn-arxiv (2.10\%). The significant improvement on Disease ($\delta=0.0$) confirms the capability of our method in modeling extreme hierarchical structures, which is consistent with the role of HSO-gated attention in emphasizing hierarchy-informative neighbors during message passing. Meanwhile, on the large-scale ogbn-arxiv, HSO-GAT not only achieves the highest accuracy but also avoids the out-of-memory (OOM) issues encountered by several Riemannian baselines (e.g., LGCN, HYBONET, and ARGNN), indicating its practical applicability to large-scale graph learning. Furthermore, consistent improvements across networks with varying $\delta$ values (from 0.0 to 3.5) indicate that HSO-GAT effectively captures diverse structural topologies, rather than being limited to specific graph geometries.

\subsection{Link Prediction}\label{Link Prediction Result}

We further evaluate HSO-GAT on link prediction, with results reported in Table~\ref{tab:results_LP}. HSO-GAT ranks first on all datasets and improves ROC-AUC by an average of 0.87 percentage points over the strongest competing method on each dataset. 

In particular, several baselines already achieve ROC-AUC scores above 98\% on Computer, Photo, WikiCS, and Airport, leaving limited room for further improvement. Nevertheless, HSO-GAT obtains the best results on these datasets (e.g., 99.19\% on Photo). More substantial gains are observed on PubMed, Disease, and ogbn-arxiv, where HSO-GAT exceeds the runner-up by 1.81, 1.54, and 1.39 percentage points, respectively. These improvements highlight that HSO-guided local curvature adaptation effectively shapes the local geometric space, yielding more discriminative geometric relations for plausible link identification in hierarchical graphs. The consistent superiority in both saturated and structurally challenging settings further shows that HSO-derived geometric guidance is effective for identifying links governed by subtle hierarchical and degree-imbalanced patterns.

\subsection{Further Analysis}\label{Further Analysis}

\subsubsection{Ablation Study}

We perform ablation experiments on PubMed and Disease (Table~\ref{tab:ablation}) to assess the contribution of each core component. Removing either HLCA or HSGA consistently degrades performance, yet with clear task-dependent patterns. On node classification, dropping HSGA hurts more than removing HLCA (e.g., Disease: 2.97 vs. 1.85 points drop), consistent with the role of attention in aggregating hierarchy-informative neighbors. On link prediction, the trend reverses: removing HLCA causes a larger degradation than removing HSGA (Disease: 1.13 vs. 0.88 points drop), confirming that curvature adaptation is more critical for preserving discriminative distance geometry in pairwise modeling. 
The gap between Base (HAT) and HSO-GAT (e.g., 9.61 points on Disease LP) further reflects the combined effect of the two modules.
Removing \(\mathcal{L}_{\mathrm{curv}}\) yields mild yet consistent degradation, confirming that curvature regularization helps stabilize training without compromising the core HSO-driven gains. Furthermore, replacing the HSO prior with the degree ratio degrades performance, because averaging degree ratios loses the dispersion information of neighborhood imbalance that HSO preserves through the convexity of its edge-level formulation.

\begin{figure}[t!]
  \centering
  \includegraphics[width=0.95\columnwidth]{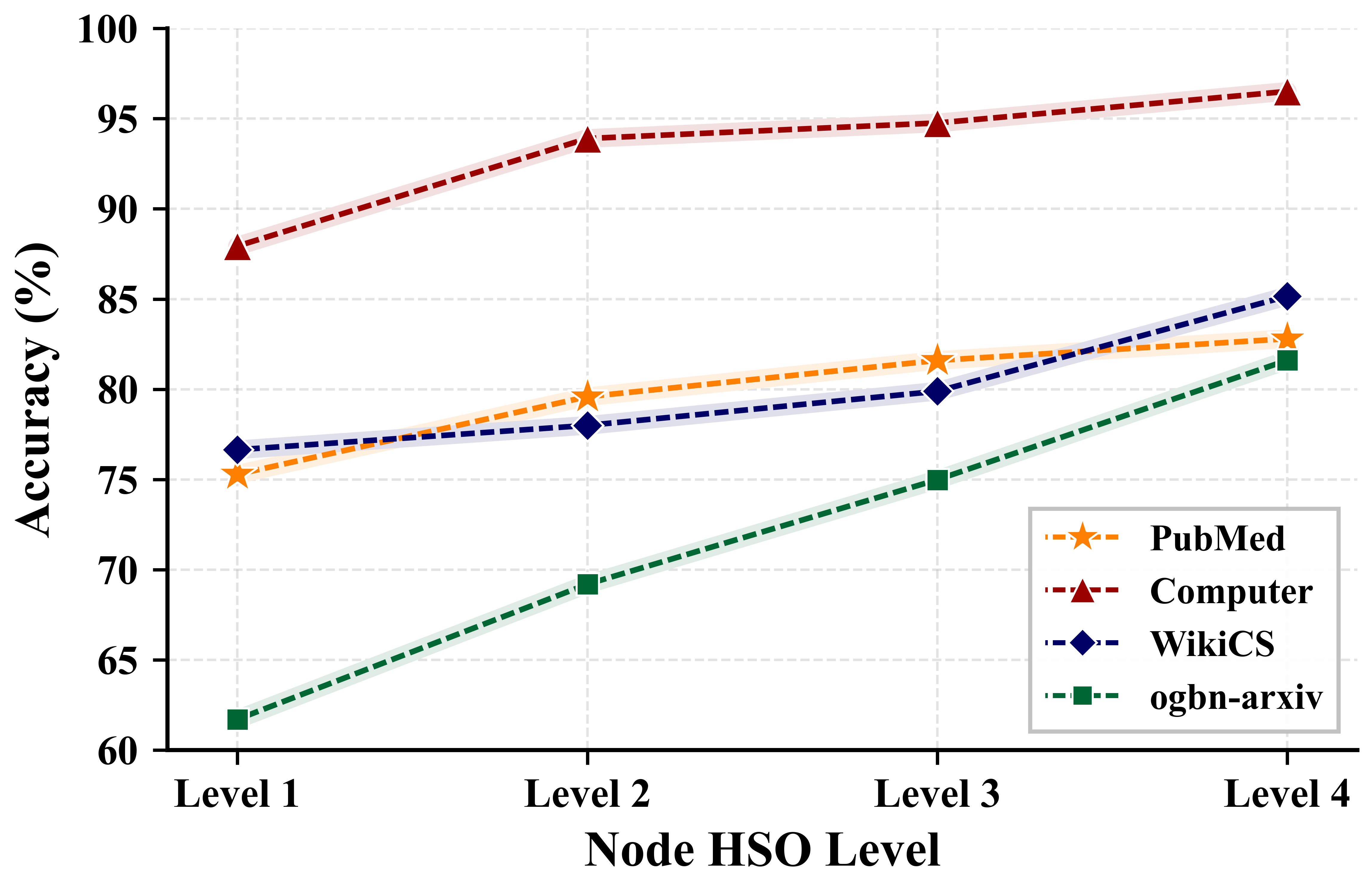}
  \caption{Node classification accuracy (\%) of HSO-GAT across different node-level HSO levels. Nodes are grouped from Level 1 to Level 4 according to increasing HSO values.}
  \label{fig:hso_quantile_trend}
\end{figure}



\subsubsection{Fine-grained Structural Analysis across HSO Levels}

To investigate the effectiveness of HSO-GAT across nodes with varying local hierarchical complexity, we partition the test nodes into four distinct HSO groups according to node-level HSO quantiles. As shown in Figure~\ref{fig:hso_quantile_trend}, all four datasets exhibit a consistent increasing trend from low to high HSO levels, where nodes with stronger hierarchy-indicative structural patterns achieve substantially better classification performance. In particular, ogbn-arxiv shows the most pronounced improvement across HSO levels, while PubMed, WikiCS, and Computer also present clear performance gains despite different degrees of saturation. This monotonic trend reveals a clear association between node-level HSO and the classification performance of HSO-GAT, which suggests that HSO offers a useful perspective for characterizing how the model adaptively allocates geometric capacity to accommodate nodes across diverse local structural regimes.

\section{Conclusion}
\label{sec:conclusion}

In this paper, we proposed HSO-GAT, a topology-adaptive hyperbolic graph attention network for hierarchical graph representation learning. By leveraging the Hyperbolic Sombor Index (HSO) as a unified structural prior, HSO-GAT couples node-wise geometric adaptation with hierarchy-aware message propagation through the complementary HLCA and HSGA modules. We further established the monotonic sensitivity of edge-level HSO to degree imbalance and analyzed the mapping validity and radial scaling properties of node-adaptive hyperbolic mappings. Extensive experiments on eight benchmarks consistently demonstrated the state-of-the-art effectiveness of HSO-GAT in both node classification and link prediction. Despite these promising results, the degree-based nature of HSO suggests that future work could explore richer topological invariants or extend the framework to dynamic graphs where hierarchical roles evolve over time.

\bibliographystyle{ACM-Reference-Format}
\bibliography{reference}

\appendix

\section{Preliminary}

\subsection{Hyperbolic Geometry}
A Riemannian manifold \cite{robbin2011introduction} $(\mathcal{M}, g)$ of dimension $n$ is a real and smooth manifold equipped with an inner product on the tangent space $g_{\mathbf{x}}: \mathcal{T}_{\mathbf{x}}\mathcal{M} \times \mathcal{T}_{\mathbf{x}}\mathcal{M} \rightarrow \mathbb{R}$ at each point $\mathbf{x} \in \mathcal{M}$, where the tangent space $\mathcal{T}_{\mathbf{x}}\mathcal{M}$ is a vector space and can be seen as a first-order local approximation of $\mathcal{M}$ around point $\mathbf{x}$. Hyperbolic space is a constant negative curvature Riemannian manifold. There are several isometric models of hyperbolic space, of which we work on the Poincar\'{e} ball model. A $d$-dimensional Poincar\'{e} ball model with curvature $-c$ ($c>0$) is defined as $\mathbb{B}_{c}^{d}=\{\mathbf{x}\in \mathbb{R}^{d}: c\Vert \mathbf{x} \Vert^2<1\}$ equipped with the Riemannian metric tensor: $g_{\mathbf{x}}^{c}=(\lambda_{\mathbf{x}}^{c})^2 g^{E}$, where $\lambda_{\mathbf{x}}^{c}:=\frac{2}{1-c\Vert \mathbf{x} \Vert^{2}}$ and $g^{E}=\mathbf{I}_{d}$ is the Euclidean metric tensor.

Since the Poincar\'{e} ball is a curved manifold, standard Euclidean vector operations (such as addition and scalar multiplication) are not closed and cannot be directly applied to points in $\mathbb{B}_c^d$. To perform algebraic operations while strictly remaining on the manifold, we adopt the framework of M\"obius gyrovector spaces. Within this framework, the standard vector addition is replaced by the M\"obius addition, defined as:
\begin{equation}
\mathbf{x} \oplus _{c} \mathbf{y} := \frac{(1+2c \langle \mathbf{x}, \mathbf{y} \rangle + c\Vert \mathbf{y} \Vert^{2})\mathbf{x} + (1-c\Vert \mathbf{x} \Vert^{2})\mathbf{y}}{1+2c \langle \mathbf{x}, \mathbf{y} \rangle + c^{2} \Vert \mathbf{x} \Vert^{2} \Vert \mathbf{y} \Vert^{2}},
\end{equation}
where $\mathbf{x},\mathbf{y} \in \mathbb{B}_{c}^{d}$. The M\"obius scalar multiplication is defined as:
\begin{equation}
r \otimes _{c} \mathbf{x} := \frac{1}{\sqrt{c}}\mathrm{tanh}\left(r\ \mathrm{tanh}^{-1}(\sqrt{c}\Vert \mathbf{x} \Vert)\right)\frac{\mathbf{x}}{\Vert \mathbf{x} \Vert},
\end{equation}
where $\mathbf{x} \in \mathbb{B}_{c}^{d} \setminus \{\mathbf{0}\}$ and $r \in \mathbb{R}$. Similarly, the M\"obius matrix-vector multiplication is defined as:
\begin{equation}
M \otimes _{c} \mathbf{x} := \frac{1}{\sqrt{c}}\mathrm{tanh}\left(\frac{\Vert M\mathbf{x} \Vert}{\Vert \mathbf{x} \Vert} \mathrm{tanh}^{-1}(\sqrt{c}\Vert \mathbf{x} \Vert)\right)\frac{M\mathbf{x}}{\Vert M\mathbf{x} \Vert},
\end{equation}
where $\mathbf{x} \in \mathbb{B}_{c}^{d} \setminus \{\mathbf{0}\}$ and $M \in \mathbb{R}^{m \times d}$. 

During the construction of hyperbolic graph neural networks, a crucial step involves the conversion between the hyperbolic space and the tangent space, facilitated by the exponential map $\mathrm{exp}_{\mathbf{x}}^{c}: \mathcal{T}_{\mathbf{x}}\mathbb{B}_{c}^{d} \rightarrow \mathbb{B}_{c}^{d}$ and the logarithmic map $\mathrm{log}_{\mathbf{x}}^{c}: \mathbb{B}_{c}^{d} \rightarrow \mathcal{T}_{\mathbf{x}}\mathbb{B}_{c}^{d}$:
\begin{equation}
\mathrm{exp}_{\mathbf{x}}^{c}(\mathbf{v}) = \mathbf{x} \oplus_{c} \left(\mathrm{tanh}\left(\sqrt{c}\frac{\lambda_{\mathbf{x}}^{c} \Vert \mathbf{v} \Vert}{2}\right)\frac{\mathbf{v}}{\sqrt{c}\Vert \mathbf{v} \Vert}\right),
\end{equation}
\begin{equation}
\mathrm{log}_{\mathbf{x}}^{c}(\mathbf{y}) = \frac{2}{\sqrt{c}\lambda_{\mathbf{x}}^{c}}\operatorname{artanh}\left(\sqrt{c}\Vert -\mathbf{x} \oplus_{c} \mathbf{y} \Vert\right)\frac{-\mathbf{x} \oplus_{c} \mathbf{y}}{\Vert -\mathbf{x} \oplus_{c} \mathbf{y} \Vert},
\end{equation}
where $\mathbf{x},\mathbf{y} \in \mathbb{B}_{c}^{d}, \mathbf{x} \neq \mathbf{y}$ and $\mathbf{v} \in \mathcal{T}_{\mathbf{x}}\mathbb{B}_{c}^{d} \setminus \{\mathbf{0}\}$. Ultimately, the definition of the hyperbolic non-linear activation incorporates the utilization of both exponential and logarithmic maps:
\begin{equation}
\sigma^{c}(\mathbf{y}) = \mathrm{exp}_{\mathbf{x}}^{c}(\sigma(\mathrm{log}_{\mathbf{x}}^{c}(\mathbf{y}))),
\end{equation}
where $\mathbf{y} \in \mathbb{B}_{c}^{d}$ and $\sigma$ is the non-linear activation in Euclidean space.

\subsection{Hyperbolic Sombor Index}\label{Hyperbolic Sombor Index}

The Sombor index is a degree-based topological index. For a graph 
$\mathcal{G}=(\mathcal{V},\mathcal{E})$, the classical graph-level Sombor index 
is defined as
\begin{equation}
\mathrm{SO}(\mathcal{G}) = \sum_{(v_i,v_j)\in\mathcal{E}}\sqrt{d_i^2+d_j^2}.
\end{equation}

The graph-level Hyperbolic Sombor Index (HSO) extends the classical Sombor index by introducing the minimum degree of the two incident nodes as a normalization factor:
\begin{equation}
\mathrm{HSO}(\mathcal{G}) = \sum_{(v_i,v_j)\in\mathcal{E}}\frac{\sqrt{d_i^2+d_j^2}}{\min\{d_i,d_j\}}.
\end{equation}

Compared with the classical Sombor index, HSO emphasizes degree imbalance between adjacent nodes through the denominator $\min\{d_i,d_j\}$, 
which makes HSO more sensitive to local hierarchical and hub-like structures.

Since $\mathrm{HSO}(\mathcal{G})$ is a summation over all edges, its edge-wise summand can be naturally used as an edge-level structural strength:
\begin{equation}
{b}_{ij} = \frac{\sqrt{d_i^2+d_j^2}}{\min\{d_i,d_j\}}.
\end{equation}

Here, $b_{ij}$ serves as the edge-level HSO prior for node-level HSO aggregation and HSO-gated attention in our model.


\section{Theoretical Proofs}

\subsection{Proof of Lemma 1}

\begin{proof}
For arbitrary edge $(v_i,v_j)\in\mathcal{E}$ satisfying $d_i\ge d_j$, substitute $\min\{d_i,d_j\}=d_j$ into the definition of edge-level HSO:
\begin{equation}
    b_{ij} = \frac{\sqrt{d_i^2+d_j^2}}{d_j} = \sqrt{\left(\frac{d_i}{d_j}\right)^2 + 1}.
\end{equation}
Let the degree ratio be $r = \frac{d_i}{d_j}$. Since $d_i \ge d_j \ge 1$, we have $r \ge 1$. 
Define the function $f(r) = \sqrt{r^2 + 1}$ for $r \in [1, +\infty)$. Taking the first derivative with respect to $r$ gives:
\begin{equation}
    f'(r) = \frac{r}{\sqrt{r^2 + 1}}.
\end{equation}
Since $r \ge 1 > 0$ and $\sqrt{r^2+1} > 0$ in the domain, it is evident that $f'(r) > 0$ holds universally. Therefore, $b_{ij}$ is strictly monotonically increasing with respect to the degree ratio $r$.
\end{proof}

\subsection{Proof of Theorem 2}
\begin{proof}
Given the tangent space representation $\mathbf{v}_i \in \mathbb{R}^d$, let its norm be $\rho = \|\mathbf{v}_i\| \ge 0$. The node-adaptive exponential mapping is defined as:
\begin{equation}
    \hat{\mathbf{h}}_i = \exp_{\mathbf{0}}^{c_i}(\mathbf{v}_i) = \frac{\tanh(\sqrt{c_i}\|\mathbf{v}_i\|)}{\sqrt{c_i}\|\mathbf{v}_i\|} \mathbf{v}_i.
\end{equation}
When $\mathbf{v}_i = \mathbf{0}$, by continuous extension of the exponential map, we have $\hat{\mathbf{h}}_i = \exp_{\mathbf{0}}^{c_i}(\mathbf{0}) = \mathbf{0}$. In this trivial case, the Mapping Validity holds since $c_i\|\hat{\mathbf{h}}_i\|^2 = 0 < 1$. 
Below, we assume $\mathbf{v}_i \neq \mathbf{0}$, i.e., $\rho > 0$. The Euclidean norm of the mapped result is:
\begin{equation}
    \|\hat{\mathbf{h}}_i\| = \frac{\tanh(\sqrt{c_i}\rho)}{\sqrt{c_i}}.
\end{equation}

\textbf{(1) Mapping Validity:} 
Calculating $c_i\|\hat{\mathbf{h}}_i\|^2$ yields:
\begin{equation}
    c_i\|\hat{\mathbf{h}}_i\|^2 = c_i \left( \frac{\tanh(\sqrt{c_i}\rho)}{\sqrt{c_i}} \right)^2 = \tanh^2(\sqrt{c_i}\rho).
\end{equation}
Since the range of the hyperbolic tangent function strictly satisfies $|\tanh(x)| < 1$ for any finite real number $x$, we have $c_i\|\hat{\mathbf{h}}_i\|^2 < 1$.

\textbf{(2) Euclidean Radial Compression:} 
Define the Euclidean absolute radius function $R(c_i, \rho) = \|\hat{\mathbf{h}}_i\| = \frac{\tanh(\sqrt{c_i}\rho)}{\sqrt{c_i}}$. 
To analyze its monotonicity with respect to $c_i$, let $t = \sqrt{c_i} > 0$. Then $R$ can be viewed as a function of $t$: $R(t) = \frac{\tanh(t\rho)}{t}$. Taking the derivative with respect to $t$:
\begin{equation}
    R'(t) = \frac{t\rho \operatorname{sech}^2(t\rho) - \tanh(t\rho)}{t^2}.
\end{equation}
Let $x = t\rho > 0$. We only need to prove that the numerator $q(x) = x\operatorname{sech}^2(x) - \tanh(x) < 0$. Taking the derivative of $q(x)$ yields:

\begin{align}
q'(x) &= \operatorname{sech}^2(x) - 2x\operatorname{sech}^2(x)\tanh(x) - \operatorname{sech}^2(x) \notag \\
      &= -2x\operatorname{sech}^2(x)\tanh(x).
\end{align}

When $x > 0$, $\operatorname{sech}^2(x) > 0$ and $\tanh(x) > 0$, so $q'(x) < 0$. Since $q(0) = 0$ and $q(x)$ is strictly decreasing, we have $q(x) < 0$ for all $x > 0$. Thus, $R'(t) < 0$, meaning the Euclidean absolute radius $R(c_i, \rho)$ is strictly monotonically decreasing with respect to $c_i$.

\textbf{(3) Normalized Radial Expansion:} 
Define the normalized radius relative to the ball boundary as $\tilde{R}(c_i, \rho) = \sqrt{c_i}\|\hat{\mathbf{h}}_i\|$. From (1), we know:
\begin{equation}
    \tilde{R}(c_i, \rho) = \sqrt{c_i} \cdot \frac{\tanh(\sqrt{c_i}\rho)}{\sqrt{c_i}} = \tanh(\sqrt{c_i}\rho).
\end{equation}
Since $\tanh(\cdot)$ is a strictly monotonically increasing function, and $\sqrt{c_i}\rho$ is strictly monotonically increasing with respect to $c_i$, their composition $\tilde{R}(c_i, \rho)$ is strictly monotonically increasing with respect to $c_i$.
\end{proof}

\subsection{Proof of Proposition 3}
By Theorem~1(1), the exponential map output
$\hat{\mathbf{h}}_i=\exp_{\mathbf{0}}^{c_i}(\mathbf{v}_i)$ satisfies
$\sqrt{c_i}\|\hat{\mathbf{h}}_i\|=\tanh(\sqrt{c_i}\|\mathbf{v}_i\|)<1$.
The projection
$\mathrm{Proj}_{\mathbb{B}_{\tilde{c}_i}}(\mathbf{x})
=\mathbf{x}/\max\!\big(1,\tfrac{\sqrt{\tilde{c}_i}\|\mathbf{x}\|}{1-\epsilon}\big)$
is norm-nonincreasing and enforces
$\sqrt{\tilde{c}_i}\|\mathbf{h}_i\|\le 1-\epsilon$.

(1) Since $\tilde{c}_i\ge c_g$, we have
$\sqrt{c_g}\|\mathbf{h}_i\|\le\sqrt{\tilde{c}_i}\|\mathbf{h}_i\|\le 1-\epsilon<1$,
i.e., $c_g\|\mathbf{h}_i\|^2<1$ and thus $\mathbf{h}_i\in\mathbb{B}_{c_g}^d$.

(2) Since $\tilde{c}_i\ge c_i$, we likewise have
$\sqrt{c_i}\|\mathbf{h}_i\|\le\sqrt{\tilde{c}_i}\|\mathbf{h}_i\|\le 1-\epsilon$.
Therefore the argument of $\operatorname{artanh}$ in
$\log_{\mathbf{0}}^{c_i}$ is bounded away from $1$ by the uniform margin
$\epsilon$, ensuring both mathematical validity and numerical stability.
$\hfill\square$

\end{document}